\documentclass[accepted]{tpm2024} %
                        \usepackage[hyphenbreaks]{breakurl}
\usepackage[american]{babel}

\usepackage{natbib} %
\usepackage{mathtools} %
\usepackage{booktabs} %
\usepackage{tikz} %

\usepackage{bm}
\usepackage{amsfonts}
\usepackage{diagbox}
\usepackage{multirow}
\usepackage{amsthm}
\usetikzlibrary{backgrounds}
\usepackage{amsmath}
\usepackage{amssymb}
\usepackage{caption}
\usepackage{subcaption}
\usepackage{scalerel}
\usepackage[capitalize,nameinlink]{cleveref}

\newtheorem{corollary}{Corollary}

\newtheorem{lemma}{Lemma}
\newtheorem{fact}{Fact}
\newtheorem{definition}{Definition}
\newtheorem{proposition}{Proposition}

\newcommand\LT{\ensuremath{\operatorname{LT}}}

\title{Probabilistic Circuits for Cumulative Distribution Functions}

\author[1]{Oliver Broadrick}
\author[1]{William Cao}
\author[1]{Benjie Wang}
\author[2]{Martin Trapp}
\author[1]{Guy Van den Broeck}
\affil[1]{%
    Computer Science Dept.\\
    University of California \\
    Los Angeles, California, USA
}
\affil[2]{%
    Department of Computer Science\\
    Aalto University\\
    Espoo, Finland
}
  
\begin{document}
\maketitle

\begin{abstract}
A probabilistic circuit (PC) succinctly expresses a function that represents a multivariate probability distribution and, given sufficient structural properties of the circuit, supports efficient probabilistic inference. Typically a PC computes the probability mass (or density) function (PMF or PDF) of the distribution. We consider PCs instead computing the cumulative distribution function (CDF). We show that for distributions over binary random variables these representations (PMF and CDF) are essentially equivalent, in the sense that one can be transformed to the other in polynomial time. We then show how a similar equivalence holds for distributions over finite discrete variables using a modification of the standard encoding with binary variables that aligns with the CDF semantics. Finally we show that for continuous variables, smooth, decomposable PCs computing PDFs and CDFs can be efficiently transformed to each other by modifying only the leaves of the circuit.
\end{abstract}

\section{Introduction}

Modeling multivariate probability distributions in a way that is both expressive and allows efficient probabilistic reasoning is a fundamental problem in the field of artificial intelligence. 
Probabilistic circuits (PCs) provide a unifying framework for a myriad of tractable probabilistic models and reduce tractability to syntactic properties of the underlying circuit \citep{diff_approach,spns,pcs}. 
In general, a PC computes a multilinear polynomial in its inputs, most commonly computing a probability mass (or density) function (PMF/PDF). 
However, other polynomials can be used to encode probability distributions and have been studied as alternative semantics for PCs, including generating functions \citep{pgcs,on_inference_pgcs,pmlr-v202-blaser23a} and characteristic functions \citep{ccs}. 
In this paper, we extend this line of work and consider the \emph{cumulative distribution function} (CDF) as a semantics for PCs. 

Unlike the PMF, which computes the probability of an input assignment, the CDF computes the probability of realizing any assignment with entries elementwise less than or equal to the input assignment.
The CDF exists %
for every real-valued multivariate probability distribution, which is not true for mass or density functions, and has broad applications in machine learning and statistics \citep{huang2009cumulative,huang2010maximum,HYVARINEN1999429,gresele2021independent}.
A particularly useful property of the CDF arises from its interpretation as a transformation, i.e., CDF transformed random variables are uniformly distributed.
This property has been heavily exploited in the literature, for example, in copulas \citep{NEURIPS2020_10eb6500}, for inverse transform sampling \citep{gentle2003random}, and density estimation via boosting \citep{awaya2021unsupervised}.

While CDFs are defined for arbitrary distributions over real-valued random variables, we consider three important special cases.
We begin with distributions over binary random variables for which PCs are known to be tractable when they express a multilinear polynomial that computes the PMF (called the \emph{PMF polynomial}). Therefore we consider multilinear polynomials that compute the CDF, which we call \emph{CDF polynomials}. We find that for binary random variables, CDF polynomials are exactly equal to probability generating functions (PGFs).
  This surprising equivalence immediately implies that a circuit computing the PMF (respectively CDF) can be transformed to a circuit computing the CDF (respectively PMF) in polynomial time, by recent results \citep{BroadrickUAI24,sanyam}. 
  Moreover, our new interpretation of PGFs as CDFs enables us to give an alternative and interesting proof of the transformation from PGFs to PMFs based on the generalized principle of inclusion-exclusion. 

Next we consider distributions over finite discrete random variables. Typically in the PC literature, such variables are handled by encoding them with binary variables, e.g., with a one-hot encoding. We find that by using a simple new encoding that respects the less-than-or-equal-to relation, we can ensure not only that the PMFs agree up to the encoding (as normal) but that also the CDFs agree up to the encoding. This allows us to reduce the finite discrete case to the binary case, applying the transformations for the binary case to obtain a similar equivalence.

Finally, we consider continuous variables, finding that for PCs which are smooth and decomposable -- standard structural properties used in the literature \citep{darwiche2002knowledge,pcs} -- simple modifications to the leaves of the PC enable transformations between PDFs and CDFs. Our results relating CDFs to PMFs/PDFs complement and extend those of \cite{BroadrickUAI24} and \cite{sanyam}, who find that several other PC semantics are equivalent for binary random variables but do not consider~CDFs.

\section{Background on circuits}\label{sec:background}

We study probabilistic circuits (PCs): computation graphs that, given sufficient structural properties, render inference tasks tractable. Let $\bm{X}=\{X_1,\ldots,X_n\}$ be random variables, and denote the set of all assignments to $\bm{X}$ by $\text{val}(\bm{X})$.

\begin{definition}
  A \emph{probabilistic circuit} (PC) in variables $\bm{X}=\{X_1,\ldots,X_n\}$ is a rooted directed acyclic graph. Each node $v$ is either (i) a product node, (ii) a sum node with edges to children labeled by weights $w_{v1},\ldots,w_{vk}\in\mathbb{R}$, or (iii) a leaf node, labeled by a function $l_v:\text{val}(X_i)\to \mathbb{R}$ for some $X_i$. 
  Each node $v$ (with children $v_1,\ldots,v_k$) computes a polynomial whose indeterminates are $l_1,\ldots,l_m$ (the functions labeling the leaves of the PC):
\begin{align*}
&p_v(l_1,\ldots,l_m)&&\\
&=\begin{cases}
\prod_{i=1}^kp_{v_i}(l_1,\ldots,l_m) &\text{if $v$ a product node}\\
\sum_{i=1}^kw_{vi}p_{v_i}(l_1,\ldots,l_m) &\text{if $v$ a sum node}\\
l_{v} &\text{if $v$ a leaf node.}
\end{cases}&&
\end{align*}
The polynomial $p$ computed by a PC is the polynomial computed by its root. The function computed by a PC is $P:\text{val}(\bm{X})\to\mathbb{R}$ given by
\[
P(\bm{x})=p(l_1(\bm{x}),\ldots,l_m(\bm{x})).
\]
Lastly, the size of a PC is the number of edges in it.
\end{definition}

If each $X_i$ is binary, taking values in $\{0,1\}$, we commonly consider the leaf functions $x_i$ and $\bar{x}_i$ respectively mapping the bit $b\in\{0,1\}=\text{val}(X_i)$ to $b$ and $1-b$. 
In the case that we only use the leaf functions $x_i$, then the polynomial computed by a PC and the function computed by a PC effectively coincide.

Note that we will typically assume the functions labeling the leaves, sometimes called \emph{input functions}, are tractable, meaning that arbitrary integrals and derivatives can be computed efficiently. 
We also note that the polynomials considered in this paper are \emph{multilinear}, meaning that they are linear in each variable. For example, the polynomials $x_1x_3-x_2x_3$ and $x_1x_2x_3+1$ are multilinear, but $x^2$ and $x_1x_2^7+1$ are not. For the remainder of this paper we will use circuit to mean PC.

\section{Cumulative Distribution Functions}\label{sec:pmf_cdf}

For real-valued random variables $X_1,\ldots,X_n$ the \emph{cumulative distribution function} (CDF) $F:\mathbb{R}^n\to [0,1]$ is
\[
F(x_1,\ldots,x_n)=\mathbb{P}[X_1\le x_1,\ldots,X_n\le x_n].
\]
This is a general notion; the CDF exists, e.g., regardless of whether the variables are discrete, continuous, or otherwise.
If the random variables are all discrete, meaning that they take values in some countable subset of the reals $D\subset \mathbb{R}$, then the distribution can also be specified by a \emph{probability mass function} (PMF) $f:D^n\to [0,1]$ given by
\[
f(x_1,\ldots,x_n)=\mathbb{P}[X_1=x_1,\ldots,X_n=x_n].
\]
If instead all the random variables are absolutely continuous, then there exists a \emph{probability density function} (PDF) 
$f:\mathbb{R}^n\to \mathbb{R}^{\ge 0}$ given by 
\[
f(x_1,\ldots,x_n)=\frac{\partial^n}{\partial x_1\ldots \partial x_n}F(x_1,\ldots,x_n).
\]
In the following sections we consider three cases. The first two concern discrete distributions: those over binary random variables in Section~\ref{sec:binary}, and those over finite discrete variables in Section~\ref{sec:finite_discrete}. We then consider continuous variables in Section~\ref{sec:continuous}.

\section{Binary Variables}\label{sec:binary}

\begin{figure}%
\center
    \begin{tabular}{ c c | c | c}
     $X_1$ & $X_2$ & $f$ & $F$ \\ 
    \hline
     0 & 0 & .1 & .1\\ 
     0 & 1 & .4 & .5\\ 
     1 & 0 & .2 & .3 \\ 
     1 & 1 & .3 & 1.0 
    \end{tabular}
    \begin{align*}
      p(x_1,x_2)&= -.2x_1x_2+.1x_1+.3x_2+.1 \tag{\cref{eq:pmfpoly}} \\
      c(x_1,x_2)& = .3x_1x_2+.2x_1+.4x_2+.1 \tag{\cref{eq:cdfpoly}} 
    \end{align*}
  \caption{A probability distribution over two binary random variables. The PMF and CDF functions are specified in the table, and the corresponding PMF and CDF polynomials $p$ and $c$ are given below. 
  }
\label{fig:example}
\end{figure}

We first consider the simplest setting: probability distributions over binary random variables. Let $X_1,\ldots,X_n$ be random variables taking values in $\{0,1\}\subset \mathbb{R}$. Then there exists a PMF $f:\{0,1\}^n\to [0,1]$ with $f(x_1,\ldots,x_n)=\mathbb{P}[X_1=x_1,\ldots,X_n=x_n]$. Moreover, there is a unique multilinear polynomial $p\in \mathbb{R}[x_1,\ldots,x_n]$ that computes $f$ (in the sense that it agrees with $f$ on all inputs in $\{0,1\}^n$) which we will call the \emph{PMF polynomial}:
\begin{equation}
  p(x_1,\ldots,x_n)=\sum_{S\subseteq [n]}f(v_S)\prod_{i\in S}x_i\prod_{i\notin S}(1-x_i)
  \label{eq:pmfpoly}
\end{equation}
where $[n]=\{1,\dots,n\}$ and $v_S\in\{0,1\}^n$ 
is the characteristic vector of $S$ ($v_S$ has $i$th entry $1$ if $i\in S$ and $i$th entry $0$ if $i\notin S$).
As an example, consider the PMF $f$ with PMF polynomial $p$ given in \cref{fig:example}.

Like the PMF, the CDF can also be uniquely expressed as a multilinear polynomial $c\in \mathbb{R}[x_1,\ldots,x_n]$ which we call the \emph{CDF polynomial}:
\begin{equation}
  c(x_1,\ldots,x_n)=\sum_{S\subseteq [n]}F(v_S)\prod_{i\in S}x_i\prod_{i\notin S}(1-x_i).
  \label{eq:cdfpoly}
\end{equation}
Note that $p$ and $c$ can be related simply:
\begin{align}\label{eq:pmf_cdf_relation}
c(x)=\sum_{y\le x}p(x)
\end{align}
where $y\le x$ is elementwise, meaning that $y_i\le x_i$ for every $i$.
Again consider the example in \cref{fig:example} with CDF $F$ and CDF polynomial $c$.

To understand how circuits computing PMF and CDF polynomials relate, we recall one additional polynomial representation for probability distributions, the \emph{probability generating function} (PGF). For binary random variables, this is the polynomial
\begin{align}\label{eq:pgf_def}
g(x_1,\ldots,x_n)=\sum_{S\subseteq [n]}f(v_S)\prod_{i\in S}x_i
\end{align}
where $v_S\in\{0,1\}^n$ is the characteristic vector of $S$.
We now make a perhaps surprising and satisfying observation. 

\begin{proposition}\label{prop:cdf_is_pgf}
Fix a probability distribution over binary random variables $X_1,\ldots,X_n$ taking values $\{0,1\}\subset \mathbb{R}$ with CDF polynomial $c$ and PGF $g$. Then $c=g$.
\end{proposition}

\begin{proof}
While $c$ and $g$ are polynomials, we abuse notation and use the same names to refer to the functions $c,g:\{0,1\}^n\to \mathbb{R}$ that they induce. Observe then that
\begin{align*}
g(x)&=\sum_{S\subseteq [n]}f(v_S)\prod_{i\in S}x_i && \\
&=\sum_{y\le 1^n}f(y)\prod_{i:y_i=1}x_i && \\
&= \sum_{y\le x}f(y) =c(x).
\end{align*}
Here the first equality follows from definition (Eq. \ref{eq:pgf_def}), the second equality follows by identifying sets with their characteristic vectors, the third equality holds because if $y> x$ (elementwise) then $\prod_{i\in S}x_i=0$, and the final equality follows from Eq. \ref{eq:pmf_cdf_relation}. Finally, the equality of the functions implies equality of the polynomials.
\end{proof}

Proposition~\ref{prop:cdf_is_pgf} immediately implies that circuits computing $p$ (respectively $c$) can be transformed to circuits computing $c$ (respectively $p$) in polynomial time\footnote{The complexity bound given in \cref{thm:main} is improved by a factor of $n$ compared to that presented in \citep{BroadrickUAI24}; this improvement follows from the use of more efficient homogenization based on polynomial interpolation \citep[Lemma~5.4]{saptharishi2015survey}.} by the results of \citet{BroadrickUAI24} and \citet{sanyam}. In other words, for binary variables circuits computing PMF polynomials and circuits computing CDF polynomials are equally succinct probabilistic models.

\begin{corollary}\label{thm:main}
Fix a probability distribution over $n$ binary random variables with PMF polynomial $p$ and CDF polynomial $c$. A circuit of size $s$ computing $p$ (respectively $c$) can be transformed to a circuit computing $c$ (respectively $p$) in time $O(ns)$.
\end{corollary}

While \citet{BroadrickUAI24} and \citet{sanyam} already show how to transform a circuit computing a PGF to a circuit computing a PMF, Proposition~\ref{prop:cdf_is_pgf} allows us to give an alternative proof, which, while essentially equivalent, provides a satisfying explanation for what the transformation is doing. Specifically, our interpretation of the PGF as a CDF allows us to view the transformation as an efficient application of the generalized principle of inclusion-exclusion. We provide the proof in \cref{appendix}.

\section{Finite Discrete Variables}\label{sec:finite_discrete}

We now consider distributions over finite discrete random variables, those taking values in some finite set of reals, $K\subset \mathbb{R}$. In particular, we focus on the case $K=\{0,1,\ldots,k-1\}$ for some $k\in\mathbb{N}$. We note that any other finite set of reals of size $k$ may be bijected with this set, as well as other categorical or ordinal random variables.

The standard approach in the probabilistic circuits literature for handling such finite discrete variables is to encode them with binary indicators, typically using a one-hot encoding. In particular, a distribution $\mathbb{P}$ over $K$-valued random variables is mapped to a distribution $\mathbb{P}'$ over binary random variables such that the two distributions are equivalent up to an injective mapping $\phi:K^n\to \{0,1\}^{kn}$. That is, for all $x\in K^n$, we have 
$\mathbb{P}(x)=\mathbb{P}'(\phi(x))$.
While the PMFs agree up to $\phi$, the same cannot be said in general for the CDFs. While both CDFs are sums over their respective PMFs, the terms in each of these sums will not in general be the same up to $\phi$ (for example, with a one-hot encoding even in the univariate case, $2\le 3$ but $(0,1,0)\nleq (0,0,1)$).

However, we observe that there is a simple alternative encoding of finite discrete variables with binary variables that does respect the elementwise less-than-or-equal-to relation.

\begin{definition}(Less-Than Encoding)
For $k\in \mathbb{N}$, the Less-Than Encoding is the function $\text{LT}_k:\{0,1,\ldots,k-1\}\to \{0,1\}^{k}$ given by 
  \[
    \LT_k(x) \coloneqq \left[\begin{matrix} \bm{1}_x \\ \bm{0}_{k-x} \end{matrix} \right]
  \]
  with $\bm{1}_n = (1, \dots, 1)^\top \in \mathbb{R}^n$ and $\bm{0}_n = (0,\dots,0)^\top \in \mathbb{R}^n$ respectively.
\end{definition}

Note that the Less-Than Encoding respects the less-than-or-equal-to relation in the sense that for any $x,y\in \{0,1,\ldots,k-1\}$, $x\le y$ if and only if $\text{LT}_k(x)\le \text{LT}_k(y)$.

We are now free to map a distribution $\mathbb{P}$ over $K$-valued random variables to a distribution $\mathbb{P}'$ over binary random variables using the Less-Than Encoding and preserving not only 
\[
p(x)=p'(\text{LT}_k(x))
\]
but also 
\[
c(x)=\sum_{y\le x}p(y) =\sum_{y\le \text{LT}_k(x)}p'(y)  =c(\text{LT}_k(x)).
\]
In particular, \cref{thm:main} can be applied to distributions over finite discrete variables using this encoding.

\section{Continuous Variables}\label{sec:continuous}

Lastly we consider the case where all of the variables are continuous with the distribution admitting a joint density function. Let $X_1,\ldots,X_n$ be real-valued random variables with CDF $F:\mathbb{R}^n\to [0,1]$ and PDF $f:\mathbb{R}^n\to \mathbb{R}^{\ge 0}$. In this setting, we show that it is straightforward to efficiently transform between circuits computing the PDF and CDF \textit{when the circuits are smooth and decomposable}. Smoothness and decomposability are standard structural properties of PCs that enable efficient inference \citep{darwiche2002knowledge,pcs}. In order to define them, we use the notion of the \emph{scope} of a node $v$ denoted $\text{scope}(v)$ which is the set of all variables appearing in the sub-PC rooted at $v$.

\begin{definition}[Smoothness]
A sum node $v$ with children $v_1, \ldots, v_k$ is smooth if the scope of its children are equal to its own scope: $\text{scope}(v)=\text{scope}(v_i)$ for $i=1,\ldots,k$.
\end{definition}

\begin{definition}[Decomposable]
A product node $v$ with children $v_1$ and $v_2$ is decomposable if the scope of its children partition its scope: $\text{scope}(v)=\text{scope}(v_1)\cup \text{scope}(v_2)$ and $\text{scope}(v_1)\cap \text{scope}(v_2)=\emptyset$.
\end{definition}

It is well known that a smooth and decomposable PC supports efficient integration, by pushing integrations to the leaves \citep{pcs}. In particular, the following proposition follows by simply replacing each input distribution $l(x)$ with a new input distribution $l_{CDF}(x)=\int_{-\infty}^x l(t) dt$.

\begin{proposition}\label{prop:pdf_to_cdf}
A smooth, decomposable PC of size $s$ computing the PDF $f$ can be transformed to a smooth, decomposable PC of size $s$ computing the CDF $F$.
\end{proposition}

In order to perform a transformation in the other direction (i.e., from CDF to PDF), decomposability alone suffices. Intuitively, partial derivatives can be pushed over sum nodes by linearity and can be pushed over product nodes because decomposability ensures that one of the two terms in the product rule is zero. A proof of the following claim is given in \cref{appendix}.

\begin{proposition}\label{prop:cdf_to_pdf}
A decomposable PC of size $s$ computing the CDF $F$ can be transformed to a decomposable PC of size $s$ computing the PDF $f$.
\end{proposition}

\section{Conclusion}

We study a basic question: what if a tractable PC computes a CDF instead of a PMF or PDF? We show that in three important cases the two models are roughly equivalent. For distributions over binary random variables, we observe that the CDF polynomial is exactly the probability generating function (PGF) of the distribution. This observation allows us to conclude that
PCs computing a PMF polynomial or CDF polynomial can be transformed to each other in polynomial time, and we were able to give a new explanation for how a circuit computing the PGF can be transformed to a circuit computing the PMF. We then showed how to reduce the finite discrete case to the binary case with a slight modification to the standard encoding with binary indicators that respects the less-than-or-equal-to relation needed to compute the CDF. Finally, we showed how the standard structural properties of smoothness and decomposability suffice in the continuous case to make transformations between PDFs and CDFs simplify to modifications of the leaves.
We leave open the question of whether similar results hold for circuit models handling discrete variables over infinite domains, mixed distributions, or finite discrete variables with more efficient encodings into binary variables, e.g. like those used by \cite{cao2023scaling} for finite discrete variables or those obtained by \citet{garg2023bit} for discretizations of continuous densities.

\begin{acknowledgements}
We thank Sanyam Agarwal for referring us to \citep{saptharishi2015survey}.
\end{acknowledgements}

\bibliography{refs}

\onecolumn

\title{Probabilistic Circuits for Cumulative Distribution Functions\\(Supplementary Material)}
\maketitle

\appendix

\section{Proofs}\label{appendix}

We first give an alternative proof that, for distributions over binary random variables, a circuit computing the PGF can be transformed to a circuit computing the PMF in polynomial time.

We identify vectors $x\in\{0,1\}^n$ with sets $S_x=\{i:x_i=1\}$ and then view $p$ and $c$ as set functions, $p,c:\mathcal{P}([n])\to\mathbb{R}$ where $\mathcal{P}([n])$ is the power set of $[n]$, yielding
\begin{equation}\label{eq:cdf_pmf_as_set_functions}
c(S)=\sum_{T\subseteq S}p(T)
\end{equation}
from Equation~\ref{eq:pmf_cdf_relation}. We now recall the generalized principle of inclusion-exclusion, which allows `inversion' of set-functions of the form in \cref{eq:cdf_pmf_as_set_functions} (see e.g. \cite{brualdi2004introductory}).

\begin{fact}[Inclusion-Exclusion]\label{fact:inclusion_exclusion}
For a finite set $S$, let $f:\mathcal{P}(S)\to \mathbb{R}$ be an arbitrary function (where $\mathcal{P}(S)$ is the power set of $S$), and let $g:\mathcal{P}(S)\to \mathbb{R}$ be given by:
\begin{align*}
g(A)=\sum_{B\subseteq A}f(B).
\end{align*}
Then,
\begin{align*}
f(A)=\sum_{B\subseteq A}(-1)^{|A|-|B|}g(B).
\end{align*}
\end{fact}

Therefore, applying \cref{fact:inclusion_exclusion} to \cref{eq:cdf_pmf_as_set_functions}, we obtain
\begin{equation}\label{eq:cdf_in_ex}
p(S)=\sum_{T\subseteq S}(-1)^{|S|-|T|}c(T).
\end{equation}
While this successfully expresses $p$ in terms of $c$, it also introduces a sum over exponentially many terms, and so any direct construction of a circuit based on this expression yields a circuit of exponential size. However it is possible to use a certain form of $c$ to compute \cref{eq:cdf_in_ex} in a single forward pass (and to construct a circuit for $c$).

For any multilinear polynomial
\begin{align*}
f(x_1,\ldots,x_n)&=\sum_{S\subseteq[n]}\alpha_S\prod_{i\in S}x_i&&
\end{align*}
with $\alpha_S\in\mathbb{R}$, we define the \emph{network}\footnote{Here the term network refers to such polynomials' origin in Bayesian Network inference \citep{diff_approach}.} form of $f$ as 
\begin{align*}
\bar{f}(x_1,\ldots,x_n,\bar{x}_1,\ldots,\bar{x}_n)&=\sum_{S\subseteq[n]}f(v_S)\prod_{i\in S}x_i\prod_{i\notin S}\bar{x}_i&&
\end{align*}
where $v_S\in\{0,1\}^n$ has $x_i=1$ for $i\in S$ and $x_i=0$ for $i\notin S$. 
Note that if we have a circuit computing $\bar{f}$, then a circuit computing $f$ can be obtained easily by replacing each $\bar{x}_i$ with $1-x_i$.
The following lemma allows us to also efficiently transform a circuit computing $f$ to a circuit computing $\bar{f}$.\footnote{The complexity bound given in \cref{lem:f_barf} is improved by a factor of $n$ compared to that presented in \citep{BroadrickUAI24}; this improvement follows from the use of more efficient homogenization based on polynomial interpolation \citep[Lemma~5.4]{saptharishi2015survey}.}
\begin{lemma}[\citet{BroadrickUAI24}]\label{lem:f_barf}
Given a circuit of size $s$ computing multilinear polynomial $f$, a circuit computing $\bar{f}$ can be constructed in time $O(ns)$.
\end{lemma}

Now, back to the problem of transforming a circuit computing $c$ to one computing $p$, we first apply \cref{lem:f_barf} to obtain a circuit computing
\begin{equation}\label{eq:cdf_bar}
\bar{c}(x,\bar{x})=\sum_{S\subseteq[n]}c(v_S)\prod_{i\in S}x_i\prod_{i\notin S}\bar{x}_i.
\end{equation}
We now observe that we can compute $p(x)$ using \cref{eq:cdf_in_ex} by evaluating $\bar{c}(y,\bar{y})$ for a carefully crafted input $(y,\bar{y})$ with entries in $\{-1,0,1\}$ (not just the typical $\{0,1\}$). In particular, we set
\begin{equation}\label{eq:def_y}
(y_i,\bar{y}_i)=\begin{cases}
(y_i,\bar{y}_i)=(0,1) & \text{if }x_i=0\\
(y_i,\bar{y}_i)=(1,-1) & \text{if }x_i=1.\\
\end{cases}
\end{equation} 
Already this provides a way to perform inference (i.e. to compute arbitrary marginal probabilities in linear time) given a circuit computing $\bar{c}$.
We can also construct a new circuit for $p$ by forming the circuit:
\begin{align*}
\bar{c}(x_1,\ldots,x_n,1-2x_1,\ldots,1-2x_n).
\end{align*}
To see that this is correct, observe
\begin{align*}
\bar{c}(x_1,\ldots,x_n,&1-2x_1,\ldots,1-2x_n) &&\\
&=\sum_{T\subseteq[n]}c(x_T)\prod_{i\in T}x_i\prod_{i\notin T}(1-2x_i) && \\
& =\sum_{T\subseteq S_x}c(x_T)\prod_{i\notin T}(1-2x_i) &&\\
& =\sum_{T\subseteq S_x}c(x_T)(-1)^{\sum_{i\notin T}x_i} &&\\
& =\sum_{T\subseteq S_x}c(x_T)(-1)^{|S_x|-|T|}&&\\%(-1)^{|S_x|}
& = p(x)&&
\end{align*}
where $S_x$ is the set with characteristic vector $x=(x_1,\ldots,x_n)$. Here the equalities hold for $x\in \{0,1\}^n$ for the following reasons. The first equality is from definition; the second equality holds because for any $T \supset S_x$, there is some $i\in T$ such that $x_i=0$, and so $\prod_{i\in T}x_i=0$; the third equality holds because $1-2x_i=(-1)^{x_i}$ for $x_i\in\{0,1\}$; the final equality holds because $T\subseteq S_x$.

We now prove \cref{prop:cdf_to_pdf}.

\begin{proof}
Given a decomposable PC computing $F$, we construct a decomposable PC for $f$ inductively. For a sum node with $P(x)=\sum_iP_i(x)$, we have 
\begin{align*}
\frac{\partial^n}{\partial x_1\ldots \partial x_n}P(x)=\sum_i\frac{\partial^n}{\partial x_1\ldots \partial x_n}P_i(x)
\end{align*}
by linearity of derivatives.
For a product node\footnote{We assume product nodes have two children; this can be enforced trivially with minimal effect on circuit size.} with $P(x)=P_1(x)P_2(x)$, we assume WLOG that $x_n$ is in the scope of $P_2$ (and therefore not $P_1$, by decomposability). Then for any $i$ we have
\begin{align*}
\frac{\partial}{\partial x_i}P(x)&=\frac{\partial}{\partial x_i}\left[P_1(x)P_2(x)\right]&&\\
&=P_1(x)\left(\frac{\partial}{\partial x_{i}}P_2(x)\right) + \left(\frac{\partial}{\partial x_{i}}P_1(x) \right) P_2(x)&&\\
&=P_1(x)\frac{\partial}{\partial x_{i}}P_2(x)&&
\end{align*}
where the final equality follows as the  partial derivative $\left(\frac{\partial}{\partial x_{i}}P_1(x) \right)$ is zero because $x_{i}$ is not in the scope of $P_1$. In such a way, we partition the partial derivatives between the children $P_1, P_2$. If $P$ is a leaf, we assume the partial derivative can be computed efficiently. Therefore taking a partial derivative of a circuit produces a circuit of the same size, and all $n$ partial derivatives can be taken while maintaining the size of the circuit.
\end{proof}

\end{document}